\documentclass[9pt]{extarticle}

\usepackage{spconf,amsmath,graphicx}
\usepackage{algorithm}
\usepackage{algpseudocode}
\usepackage{amssymb}
\usepackage{amsthm}
\usepackage{xcolor}
\usepackage{subcaption}
\usepackage{enumitem}
\usepackage{hyperref}
\newtheorem{theorem}{Theorem}
\newtheorem{lemma}{Lemma}

\newtheorem{assumption}{Assumption}
\makeatother
\algrenewcommand\algorithmicrequire{\textbf{In:}}
\algrenewcommand\algorithmicensure{\textbf{Out:}}
\algrenewcommand\algorithmicindent{0.8em}

\newcommand{\ones}{\mathbf{1}}
\allowdisplaybreaks

\begin{document}

\title{\textbf{FedAVOT}: Exact Distribution Alignment in \\Federated Learning via Masked Optimal Transport\vspace{-8pt}}
\markboth{Journal of \LaTeX\ Class Files,~Vol.~14, No.~8, August~2015}%
{Shell \MakeLowercase{\textit{et al.}}: Bare Demo of IEEEtran.cls for IEEE Journals}
\maketitle

\begin{abstract}
\vspace{-3pt}
Federated Learning (FL) allows distributed model training without sharing raw data, but suffers when client participation is partial. In practice, the distribution of available users (\emph{availability distribution} $q$) rarely aligns with the distribution defining the optimization objective (\emph{importance distribution} $p$), leading to biased and unstable updates under classical FedAvg. We propose \textbf{Fereated AVerage with Optimal Transport (\textbf{FedAVOT})}, which formulates aggregation as a masked optimal transport problem aligning $q$ and $p$. Using Sinkhorn scaling, \textbf{FedAVOT} computes transport-based aggregation weights with provable convergence guarantees. \textbf{FedAVOT} achieves a standard $\mathcal{O}(1/\sqrt{T})$ rate under a nonsmooth convex FL setting, independent of the number of participating users per round. Our experiments confirm drastically improved performance compared to FedAvg across heterogeneous, fairness-sensitive, and low-availability regimes, even when only two clients participate per round.
\end{abstract}
\vspace{-4pt}
\begin{keywords}
Federated Learning, Optimal Transport, Partial Participation, Convergence, Fairness.
\end{keywords}
%
\vspace{-9pt}
\section{Introduction and Problem Setup}
\label{sec:intro}
\vspace{-7pt}

Federated Learning (FL) has emerged as a decentralized paradigm for training machine learning models across multiple clients without requiring direct access to their raw data~\cite{mcmahan2017communication, kairouz2021advances}. In this framework, each client computes local updates on its private dataset and communicates only model parameters or gradients to a central server, which then aggregates these updates to form a global model. This design ensures privacy preservation and compliance with data protection regulations, while enabling large-scale collaboration across data silos. Consequently, FL has been widely applied in privacy-sensitive domains such as healthcare, finance, and personalized recommendation systems~\cite{yang2019federated, bonawitz2019towards}.

Despite these advantages, the deployment of FL in practice is hindered by several challenges. First, clients may have intermittent connectivity, variable availability, or limited computational resources~\cite{bonawitz2019towards}. Second, data across clients is rarely independent and identically distributed (IID), but instead exhibits strong heterogeneity, leading to significant optimization and generalization difficulties~\cite{li2020federated, zhao2018federated}. Third, the number of active clients per communication round is often severely restricted, either due to network limitations or user participation constraints. These limitations fundamentally alter the optimization dynamics relative to centralized training.

A critical but underexplored issue arises when distinguishing between two distinct distributions in FL:  
(i) the \emph{availability distribution}, which governs how often each user participates in training, and  
(ii) the \emph{importance distribution}, which characterizes the relative contribution of each user's data to the global optimization objective.  
Standard algorithms such as FedAvg implicitly assume these distributions are aligned, or more restrictively, that user data should be weighted uniformly~\cite{mcmahan2017communication, kairouz2021advances}. In practice, however, this assumption rarely holds: users with high availability may possess uninformative or redundant data, while infrequent participants may hold data that is disproportionately important for the global model~\cite{chen2020optimal, mohri2019agnostic, li2019fair, li2020federated, rahimi2025agnosticfedavg, dio2024restricted, li2020convergence}. This misalignment is further exacerbated by data heterogeneity~\cite{zhao2018federated, karimireddy2020scaffold}, skewed participation~\cite{wang2020tackling, cho2022towards}, and fairness considerations~\cite{li2019fair, du2021fairness}, all of which can significantly impact both convergence guarantees and model performance. Neglecting this discrepancy can therefore lead to systematic bias, instability, and degraded performance in partial participation regimes\cite{rizk2021importance, ribero2022intermittent, wang2022fedgs}.


\textbf{\textit{Local and Global Objectives.}}  
Formally, let the input space be \( \mathfrak{X} \subset \mathbb{R}^d \) and the label space be \( \mathcal{L} = \{1, \dots, L\} \). Each client \( i \in [N] \) holds a local dataset 
$
D_i = \{(X_i^j, Y_i^j)\}_{j=1}^{n_i},(X_i^j, Y_i^j) \sim \mathcal{D}_i,
$
and minimizes its empirical risk
$
\label{equation:local-objective}
    f(\theta; D_i) = \frac{1}{n_i} \sum_{j=1}^{n_i} \ell(m(X_i^j; \theta), Y_i^j),
$
where \( \theta \in \Theta \subseteq \mathbb{R}^{d'} \) are the model parameters, \( m: \mathfrak{X} \times \Theta \to \mathcal{L} \) is the predictor, and \( \ell: \mathcal{L}\times \mathcal{L} \to \mathbb{R}_+ \) is a standard loss function (e.g., cross-entropy). Clients typically optimize the aforementioned local objetives via stochastic gradient descent (SGD) and, if available for communication, transmit their updated parameters to the server. The global optimization problem is then
\begin{equation}
    \label{equation:global-objective}
    F(\theta) := \sum_{i=1}^N p_i f_i(\theta),
\end{equation}
where $f_i(\cdot) := f(\cdot; D_i)$ and $(p_i)_{i=1}^N$ is a user-weighting distribution that we call the \emph{importance distribution}. This distribution may encode fairness criteria~\cite{li2019fair, mohri2019agnostic}, robustness to minority populations, or business-driven objectives.

\textbf{\textit{Availability vs. Importance Distributions.}}  
In practice, clients do not always participate. At each round $t$, a random subset $S^t \subseteq [N]$ of clients becomes available. We model this by a distribution $q$, where $q(\mathcal{A})$ is the probability of observing client set $\mathcal{A} \subseteq [N]$. 
Without loss of generality, we assume that $q$ is supported on $\{\mathcal{A}_j\}_{j\in[M]}$, for $M<2^N$.
Thus, optimization dynamics are governed by $q$, not $p$. 
This distinction between the \emph{availability distribution} $q$ and the \emph{importance distribution} $p$ (defined above) has been largely neglected in the literature, despite its centrality to fairness and robustness in FL.

The server aggregates parameters from available clients via some aggregation rule to update the global model $\hat{\theta}^t$. FedAvg~\cite{mcmahan2017communication} performs
$    \hat{\theta}^t = {1}/{|S^t|}\sum_{i \in S^t} \theta_i^t,
$
which is itself equivalent to optimizing a \textit{surrogate problem} \cite{rahimi2025agnosticfedavg}
\begin{equation}
    \tilde{F}(\theta) = \sum_{i=1}^N \tilde{p}_i f_i(\theta), \quad \text{where}\quad
    \tilde{p}_i = \sum_{\mathcal{A} \ni i} \frac{q(\mathcal{A})}{|\mathcal{A}|}. \label{eq:importance-implied-by-availability}
\end{equation}
Hence, \textit{unless $q$ aligns with $p$} (e.g., uniform participation and importance), FedAvg converges to a minimizer (or a stationary point) of $\tilde{F}$ rather than $F$ in~\eqref{equation:global-objective}~\cite{li2020federated}. Moreover, heuristic aggregation rules, such as the commonly used weighting $(N/|S^t|)p_i$, fail to guarantee convergence in general non-uniform partial participation regimes\cite{condat2025stochasticmultiproximalmethodnonsmooth}.

More specifically, it is well-known that if all devices participate in every round ($q([N])=1$), then any target distribution $p$ can be achieved exactly via the weighted update rule
$    \hat{\theta}^t = \sum_{i=1}^N p_i \theta_i^{t-1}$.
However, under partial participation, achieving convergence to $F(\theta)$ in~\eqref{equation:global-objective} requires an aggregation policy that systematically accounts for both $q$ and $p$. The central question we consider is therefore:
\vspace{-3pt}
\begin{quote}
\centering
\emph{\textbf{Can we adapt client aggregation so that optimization governed by the availability distribution $q$ converges according to the target distribution $p$~(cf. \eqref{equation:global-objective})?}}
\end{quote}
\vspace{-3pt}
In this paper, we develop a novel algorithm called \textbf{Federated Averaging with Optimal Transport (\textbf{FedAVOT})}, introducing a new aggregation paradigm for FL that \textit{aligns} the availability distribution $q$ (governing which users participate in aggregation) with the importance distribution $p$ (defining the optimization   objective in \eqref{equation:global-objective}) \textit{via a masked optimal transport construction}. This perspective exposes structural gaps in existing FL formulations and leads to both theoretical and practical advances. Our key contributions are as follows: \vspace{-12pt}\setlist[enumerate]{leftmargin=1.4em}
\begin{enumerate}
    
    \item \textit{\textbf{Feasibility via Flow Duality.}} We reduce the feasibility of \textit{exact} distribution alignment (i.e., existence of a \textit{feasible} transport map) to a \emph{max-flow/min-cut problem} on a bipartite graph, and establish \textit{necessary and sufficient conditions in the form of Hall-type inequalities}~\cite{hall1935representatives, ford1956max, fulkerson1962flows}. To the best of our knowledge, this is the first work to identify such a precise combinatorial characterization in the context of FL aggregation to address distribution shift. 
    
    \vspace{-4pt}    
    \item \textit{\textbf{Algorithm and Convergence.}} We introduce the \textbf{FedAVOT} algorithm, which modifies the aggregation step using feasible transport plans, and establish convergence of standard order $\mathcal{O}(1/\sqrt{T})$ under a relaxed convex setting, matching the best-known (and optimal) rates for FL~\cite{li2020federated}. Crucially, these guarantees hold under the exact feasibility conditions discussed above. \vspace{-4pt}
    
    \item \textit{\textbf{Independence from Aggregation Size.}} Quite surprisingly (at least at first), we show both theoretically and empirically that \textbf{FedAVOT} achieves \textit{the same convergence rate regardless of the number of available clients in each round}. The same bound holds even when aggregation is performed with as few as \emph{two users per round}, a regime where standard (weighted) FedAvg fails~\cite{bonawitz2019towards, zhao2018federated}. \vspace{-4pt}
    
    \item \textit{\textbf{Empirical Validation.}} We validate \textbf{FedAVOT} on diverse tasks, including coordinated sampling (i.e., server-controlled), fairness-aware FL~\cite{mohri2019agnostic}, and restricted availability settings~\cite{chen2020optimal}. Across all scenarios, \textbf{FedAVOT} consistently improves upon FedAvg in the partial participation setting in both stability and final accuracy and achieves the (or near-) same performance of full device participation with an aggregation size of just two users per round. \vspace{-4pt}
\end{enumerate}

\vspace{-14pt}
\section{FedAVOT}
\label{sec:FedAVOT}
\vspace{-7pt}

A fundamental challenge in FL under partial participation is reconciling the \emph{availability distribution} $q$---which governs how frequently clients participate in training---with the \emph{importance distribution} $p$ that defines the global optimization objective~\cite{mcmahan2017communication, li2020federated, kairouz2021advances}. Classical FedAvg implicitly assumes that each user is on average available to the server proportional to its importance\footnote{As shown in our previous paper \cite{} this means that either uniform importance and availability is assumed, or $\tilde{p} = p$ where $\tilde{p}$ is the marginal of form \eqref{eq:importance-implied-by-availability} of $q$.} or that $q$ is uniform, leading to updates that converge towards a biased surrogate objective~\cite{mcmahan2017communication, mohri2019agnostic}. To address this mismatch, we propose \textbf{Federated AVerage with Optimal Transport (\textbf{FedAVOT})}, which formulates the aggregation step as a constrained optimal transport (OT) problem aligning $q$ and $p$.

At each communication round $t$, instead of assigning uniform weights $1/{|S^t|}$ to available clients, we would like to assign weights proportional to their target importance probabilities $p_i$. Let $Y[i,j]$ be the normalized contribution of client $i$ (in aggregation) when the active set of clients is $\mathcal{A}_j(=S^t)$, and define $T[i,j]=q_j Y[i,j]$ \textit{as the joint allocation of mass from event $j$ to client $i$}. Hence, we have implicitly defined a mask on $Y$ such that users that have not participated in a communication round should \textit{not} be assigned a weight. By construction, $T$ must satisfy the marginal and feasibility constraints: 
\begin{align}\label{equation: constraints}
    &\text{Row sums: } \sum_{j} T[i,j] = p_i, && \forall i\in [N], \nonumber\\
    &\text{Column sums: } \sum_{i} T[i,j] = q_j, && \forall j\in [M], \nonumber\\
    &\text{Masking: } T[i,j]=0, &&  i\notin \mathcal{A}_j, \nonumber\\
    &\text{Nonnegativity: } T[i,j]\ge0, && \forall (i,j). \tag{C1-C4}
\end{align}
Equivalently, we may write
\begin{align}
    T \mathbf{1}_M = p, \quad \mathbf{1}_N^\top T = q, \quad T \ge 0, \tag{MOT}\label{MOT}
\end{align}
with support restricted to the \textit{mask} $\mathcal{E}=\{(i,j): i\in \mathcal{A}_j\}$.

This construction reduces client aggregation to a \emph{masked optimal transport} feasibility problem: transporting mass $q$ over subsets $\mathcal{A}_j$ to mass $p$ over users, under feasibility restrictions. Unlike classical OT where the cost of transport matters, here we just need to find one (out of possibly many) transport map that respects the marginals and conditions and we are not trying to optimize a cost function~\cite{villani2008optimal, peyre2019computational}. Therefore, we can go for a trivial (in)feasiblity indicator cost function over the mask constraints 
in $\mathcal{E}$.
Defining the cost matrix $C \in \left(\mathbb{R}\cup\{\infty\}\right)^{N\times M}$ as $C[i,j]=1$ if $i\in\mathcal{A}_j$ and $C[i,j]=\infty$ if not,
\eqref{MOT} is equivalent to the \textit{Kantorovich relaxation}~\cite{kantorovich1942transfer, villani2008optimal}:
\begin{align}
    \min_{T \ge 0} \ \sum_{i=1}^N\sum_{j=1}^M C[i,j]T[i,j] 
    \quad \text{s.t. } T\mathbf{1}_M = p,\ \mathbf{1}_N^\top T = q. \label{eq:kantorovich}
\end{align}
Remarkably, existence of a feasible masked transport map is fully characterized by a max-flow/min-cut argument. As the next result suggests, feasibility of \eqref{eq:kantorovich} reduces to a subset of inequalities, a direct generalization of Hall’s condition~\cite{hall1935representatives} and the Ford–Fulkerson theorem~\cite{fulkerson1962flows}. 
This reduction also underscores the key distinction from  standard OT: instead of a full-blown geometric optimization problem, \eqref{MOT} may be seen as a combinatorial feasibility problem governed by flow-cut duality.
\vspace{-2pt}
\begin{theorem}[\textbf{Feasibility of MOT} \cite{villani2009optimal,peyre2019computational}]
\label{thm:feasibility}
For $I\subseteq[N]$, let $\mathcal{N}(I) := \{j\in [M]: \exists i\in I \ \text{with}\ (i,j)\in \mathcal{E}\}$ and $\mathcal{S}(I):=\{j\in [M]: \mathcal{A}_j \subseteq I\}$. Then \eqref{MOT} is feasible if and only if
\begin{equation}\label{equation-feasibility}\tag{Feasibility}
    \sum_{j\in \mathcal{S}(I)} q_j \;\leq\; \sum_{i\in I} p_i \;\leq\; \sum_{j\in \mathcal{N}(I)} q_j, \quad \forall I\subseteq [N].
\end{equation}
\end{theorem}
\begin{proof}[Sketch]
Construct a bipartite flow network with source $s$, clients $i\in[N]$ with supply $p_i$, availability nodes $j\in[M]$ with demand $q_j$, and sink $t$. Edges $s\to i$ and $j\to t$ have capacities $p_i$ and $q_j$, while $i\to j$ edges exist iff $i\in A_j$ with infinite capacity. By the max-flow/min-cut theorem feasibility holds iff \eqref{equation-feasibility} is satisfied~\cite{ford1956max, fulkerson1962flows}.
\end{proof}
\vspace{-5pt}
Theorem~\ref{thm:feasibility} is central to our framework. 
Despite its brief proof, it provides a complete characterization of feasibility for MOT via the subset inequalities~\eqref{equation-feasibility}, which are both necessary and sufficient. 
This condition generalizes Hall’s classical matching theorem~\cite{hall1935representatives} and the cut conditions of max-flow/min-cut~\cite{ford1956max,fulkerson1962flows}, and is recognized in modern optimal transport theory~\cite{villani2009optimal,peyre2019computational}. 

Given feasibility, $T$ can be computed by the \textit{iterative proportional fitting procedure} (IPFP)~\cite{sinkhorn1967concerning, knight2008sinkhorn, peyre2019computational}, also known as \textit{Sinkhorn scaling}, which alternates between row and column rescaling.  
\begin{algorithm}[t]
\caption{Sinkhorn Scaling (to find optimal plan $T$)}
\label{alg:solving-for-t}
\begin{algorithmic}[1]
\Require $p,q,\mathcal{E},\varepsilon$
\State \textbf{Init:} $T^{(0)}\!\ge0$, $\operatorname{supp}(T^{(0)})\subseteq\mathcal{E}$, and $T^{(0)\top}\ones=q$\footnote{$(T^{(0)})[i, S] = \frac{1}{|S|} \forall i \in S$}.
\For{$t=0,1,2,\ldots$}
  \State $r \gets p \oslash (T^{(t)}\ones)$; \quad $\tilde T \gets \operatorname{Diag}(r)\,T^{(t)}$
  \State $c \gets q \oslash (\tilde T^\top\ones)$; \quad $T^{(t+1)} \gets \tilde T\,\operatorname{Diag}(c)$
  \If{$\|T^{(t+1)}\ones - p\|_{1} \le \varepsilon$ and $\|\ones^\top T^{(t+1)} - q^\top\|_{1} \le \varepsilon$} \textbf{stop}
  \EndIf
\EndFor
\State \textbf{Output:} $T^{(t+1)}$, normalized weights $Y \gets T^{(t+1)}\,\mathrm{Diag}(q)^{-1}$
\end{algorithmic}
\end{algorithm}
\vspace{-5pt}
\begin{theorem}[\textbf{Convergence of IPFP \cite{sinkhorn1967concerning, knight2008sinkhorn}}]
\label{thm:convergence}
If $p$ and $q$ satisfy \eqref{equation-feasibility}, Algorithm~\ref{alg:solving-for-t} ($\varepsilon = 0$) converges to a solution of \eqref{MOT}.
\end{theorem}
\vspace{-14pt}
\begin{proof}[Sketch]
Classical IPFP analysis~\cite{sinkhorn1967concerning, knight2008sinkhorn} applies.
We consider the entropy-regularized OT formulation~\cite{cuturi2013sinkhorn} that is strictly convex with a minimizer in the set of minimizers of \eqref{eq:kantorovich} by adding the the term $\sum_{i,j} T_{ij}\log{T_{ij}}$ to problem \eqref{eq:kantorovich}, whose dual is strictly concave with a unique maximizer\footnote{The regularization term is nothing but Shannon Entropy of the Transport map.}. IPFP corresponds to block-coordinate ascent on the dual~\cite{peyre2019computational}, yielding monotone convergence. 
\end{proof}
\vspace{-8pt}

Although conditions~\eqref{equation-feasibility} are necessary and sufficient, checking them explicitly can be computationally prohibitive. However, IPFP still converges to a unique minimizer of the corresponding entropy-regularized loss~\cite{cuturi2013sinkhorn}, guaranteeing that at least one marginal constraint (rows or columns) is exactly satisfied, while the other is projected to the closest feasible distribution in KL divergence~\cite{csiszar1975ipr}. A detailed characterization of this behavior will be provided in an journal version of this paper.

\vspace{-10pt}
\section{Convergence Analysis of \textbf{FedAVOT}}
\vspace{-7pt}
\label{sec:convergence}

\begin{algorithm}[t]
\caption{\textbf{FedAVOT}}
\label{alg:agnostic-fedavg}
\begin{algorithmic}[1]
\Require $\theta^{-1}\!=\!\mathbf{0}$, $S$, $H$, $\eta_\theta > 0$,  $\mathcal{C} \subset \Theta$; $p$, $q$, $\{\mathcal{A}_j\}$, $\mathcal{E}$, $\varepsilon >0$
\State Compute $(T,Y)$ via Alg.~\ref{alg:solving-for-t} for $(p,q,\mathcal{E}, \varepsilon)$
\For{$t=0,1,\ldots,\, SH-1$}
  \If{$t \bmod H = 0$} \Comment{\textbf{\textit{Global Communication}}}
    \State Observe active set $S^t\!\subseteq\![N]$ (let $j(t)$ be s.t. $\mathcal{A}_{j(t)} = S^t$)
    \State \textbf{aggregate} $\hat\theta^{\,t} \gets \sum_{i \in S^t} Y[i, j(t)]\,\theta_i^{\,t-1}$
    \State \textbf{broadcast} $\theta_i^{\,t} \gets \hat\theta^{\,t}$ for all $i\in [N]$
  \Else \Comment{\textbf{\textit{Local Updates}}}
    \State $\theta_i^{\,t} \gets \Pi_{\mathcal{C}}\!\left(\theta_i^{\,t-1} - \eta_\theta\, \nabla_\theta f_i(\theta_i^{\,t-1};\,\xi_i^{\,t})\right)$ \quad for all $i\in [N]$
  \EndIf
\EndFor
\Ensure $\frac{1}{S}\sum_{\tau=1}^{S} \hat\theta^{\,\tau H}$
\end{algorithmic}
\end{algorithm}

We establish convergence of \textbf{FedAVOT} (see Alg. \ref{alg:agnostic-fedavg}) under a nonsmooth convex and bounded variance setting. The algorithm involves two sources of randomness:  
(i) for each user $i \in [N]$, $\xi_i^t$ denotes a mini-batch of size $b$ sampled without replacement from $D_i$ at round $t$;  
(ii) $S^t \subseteq [N]$ is the random set of active users at round $t$, sampled \textit{iid}\ according to $q$~\cite{mcmahan2017communication,li2020federated}.  
In global rounds only $S^t$ matters, while in local rounds only $\xi_i^t$ is relevant. Let $\{\mathcal{F}_t\}_t$ be the filtration
\vspace{-3pt}
\[
\mathcal{F}_t := \sigma\!\left(\theta_i^s,\,\xi_i^s,\,S^s : s\le t,\, i\in[N]\right),
\]
capturing model states, mini-batch selections, and participation history. Then $\{\theta_i^t\}$ evolves as a Markov process with respect to $\{\mathcal{F}_t\}$. \vspace{-12.5pt}
\begin{assumption}[\textbf{Convexity}]
\label{assump:convex}
Each $f_i(\cdot):=f(\cdot;D_i)$ is convex.
\end{assumption}


\vspace{-6pt}
\begin{assumption}[\textbf{Bounded Gradient Norm}]
\label{assump:gradient_bound}
For all $i\in[N]$, it is true that
$
\sup_{\theta\in\mathcal{C}}\mathbb{E}_{\xi_i}\!\left[\|\nabla f(\theta;\xi_i)\|^2\right]\le G^2.
$
\end{assumption}
\vspace{-1pt}

We further tacitly assume $\mathcal{C}$ is convex compact so that projection $\Pi_{\mathcal{C}}(\cdot)$ is well defined, and that $\theta^*\in\mathcal{C}$ solves \eqref{equation:global-objective}. Lipschitz continuity of $f_i$ on $\mathcal{C}$ follows from convexity and compactness. The next result is crucial in deriving a convergence rate for \textbf{FedAVOT}. 

\begin{lemma}[\textbf{Sample-to-Model Inequality}]
\label{lem:sample-to-model}
At every global round $t$,
\[
\mathbb{E}\!\left[\|\hat{\theta}^t-\theta^*\|^2\mid\mathcal{F}_{t-1}\right]
\;\le\;
\sum_{i\in[N]}p_i\,\|\theta_i^{t-1}-\theta^*\|^2.
\]
\end{lemma}

\begin{figure*}[!ht]
    \centering
    \begin{subfigure}{0.49\linewidth}
        \includegraphics[width=\linewidth,height=0.32\textheight,keepaspectratio,clip]{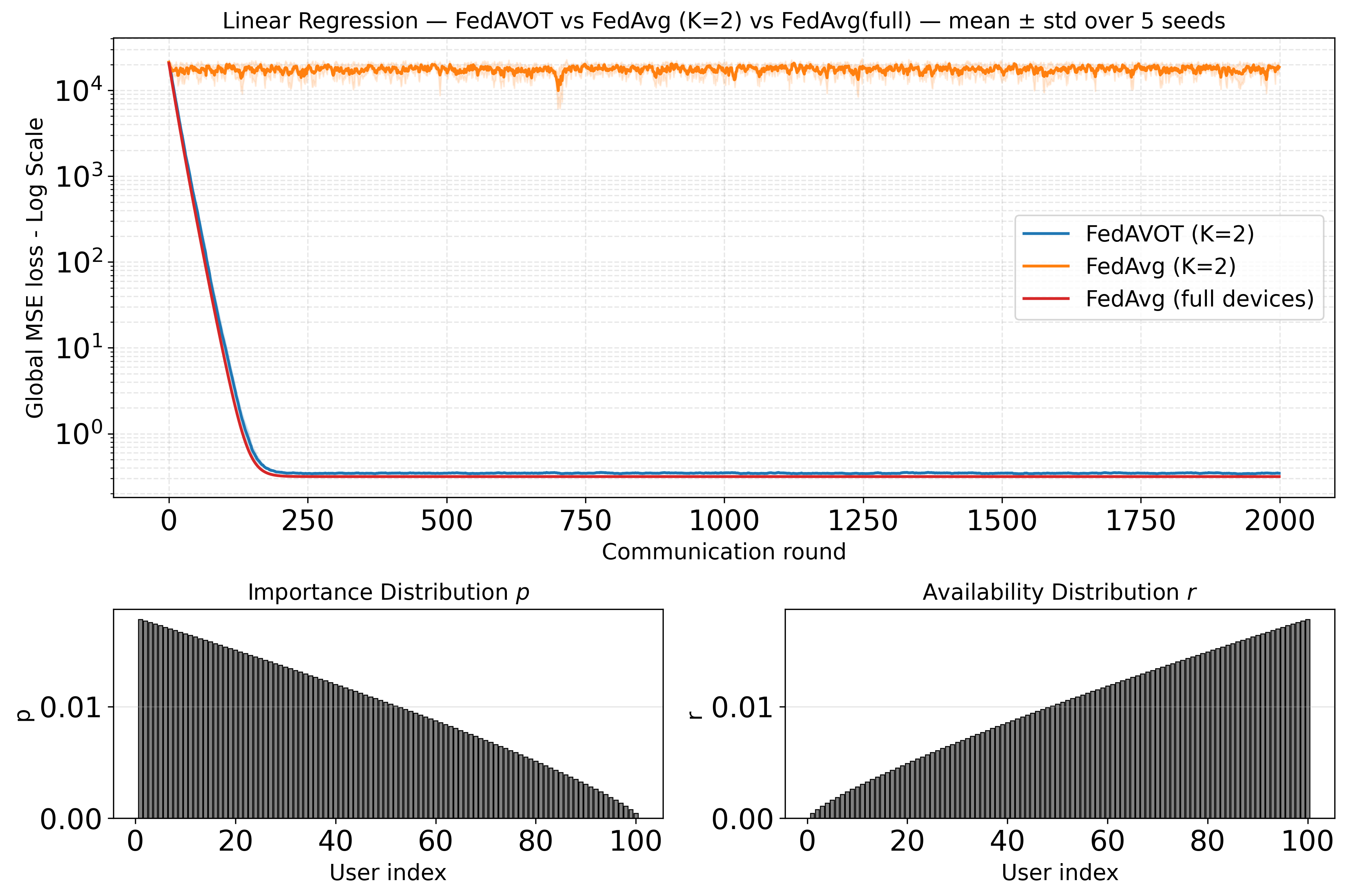}
        \vspace{-14pt}
        \caption{Linear regression under restricted availability (Log Scale).}
        \label{fig:linreg}
    \end{subfigure}
    \hfill
    \begin{subfigure}{0.49\linewidth}
        \includegraphics[width=\linewidth,height=0.32\textheight,keepaspectratio,clip]{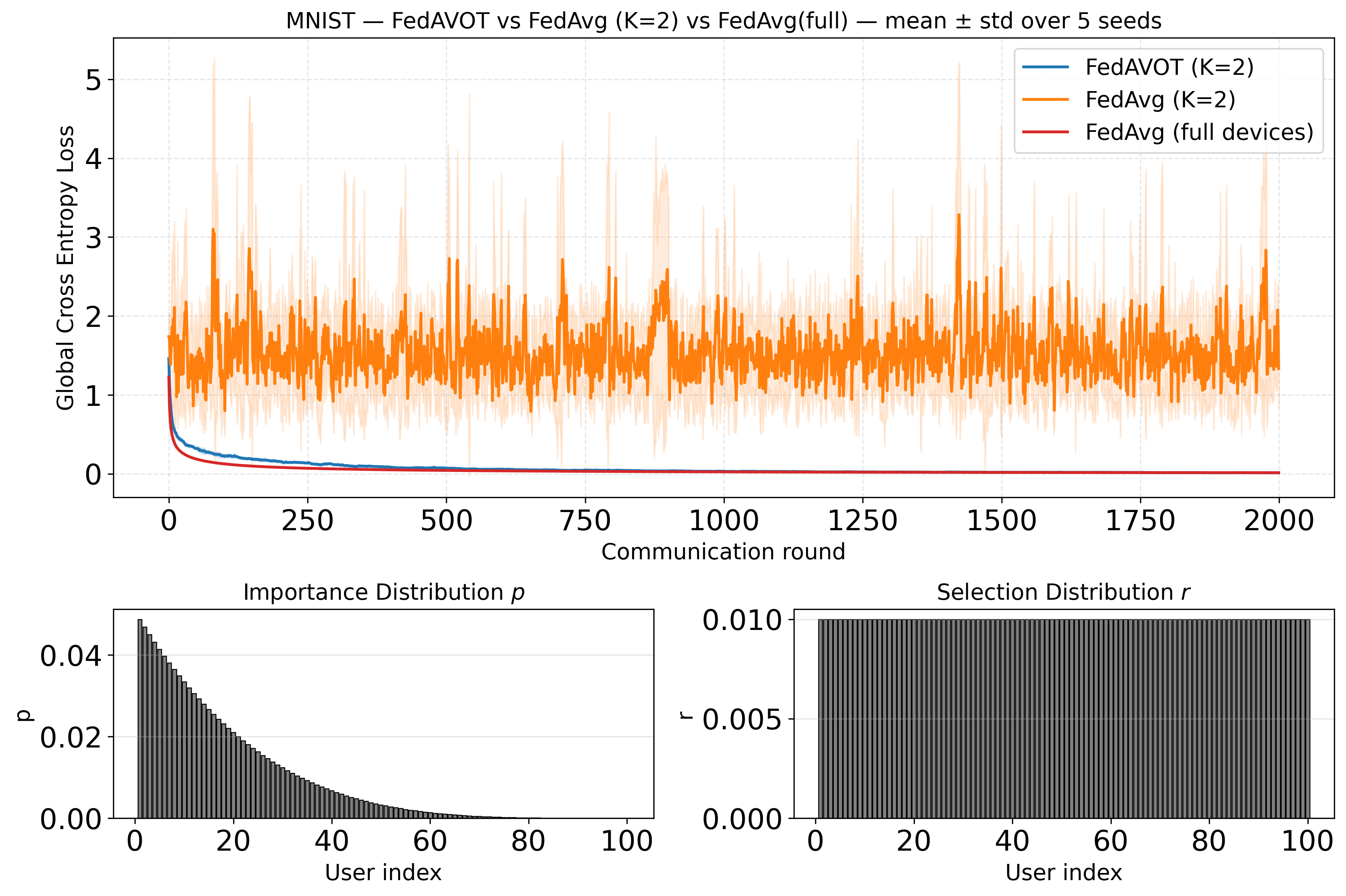}
        \vspace{-14pt}
        \caption{Multinomial logistic regression on MNIST (coordinated setting).}
        \label{fig:mnist}
    \end{subfigure}
    \vspace{-6pt}
    \caption{Comparison of \textbf{FedAVOT} and FedAvg baselines. Shaded regions show mean $\pm$ std.\ over $5$ seeds.}
    \vspace{-12pt}
    \label{fig:experiments}
\end{figure*}
\vspace{-12pt}
\begin{proof}
First, for $j(t)\equiv j(S^t)$ (see Alg. \ref{alg:agnostic-fedavg}), Jensen implies that
\begin{align}\nonumber
\mathbb{E} \hspace{-2pt} \left[\|\hat{\theta}^{t} - \theta^*\|^2\mid \mathcal{F}_{t-1}\right] &\hspace{-3pt}=\hspace{-1pt}  \mathbb{E}_{S^{t}}
\hspace{-2pt}\Bigg[ \Bigg\|  \sum_{i \in S^{t}} Y[i, j(t)]\theta_i^{t-1} - \theta^* \Bigg\|^2 \Bigg| \mathcal{F}_{t-1}\Bigg] \\ \nonumber
&\hspace{-2pt}\leq \hspace{-1pt}   \mathbb{E}_{S^{t}} 
\hspace{-2pt}\Bigg[  \sum_{i \in S^{t}} Y[i,j(t)] \|\theta_i^{t-1} - \theta^*\|^2 \Bigg| \mathcal{F}_{t-1}\Bigg]\hspace{-1pt}.
    \end{align}
Expanding the expectation on the right-hand side, we have
\begin{align}\label{GlobalToLocal}\tag{G.T.L.}
&\hspace{-6pt}\mathbb{E} \left[\|\hat{\theta}^{t} - \theta^*\|^2\mid \mathcal{F}_{t-1}\right] \\ \nonumber
&\leq \sum_{\mathcal{A}} \bigg[\mathbb{P}[S_{t}=\mathcal{A}] \cdot  \sum_{i \in [N]} Y[i,j(\mathcal{A})]\mathbb{I}[i\in \mathcal{A}] \|\theta_i^{t-1} - \theta^*\|^2\bigg]
 \\ \nonumber
&=  \sum_{i \in [N]} \left(\sum_{\mathcal{A}}Y[i, j(\mathcal{A})]\mathbb{P}[S_{t}=\mathcal{A}]\mathbb{I}[i \in \mathcal{A}]\right) \cdot \|\theta_i^{t-1} - \theta^*\|^2\\ \nonumber
&=  \sum_{i \in [N]} p_i \cdot  \|\theta_i^{t-1}  - \theta^*\|^2. 
\end{align}
and we are done.
\end{proof}
\vspace{-4pt}
\noindent The convergence rate of \textbf{FedAVOT} can now be established; the bulk of the analysis is omitted, but resembles \cite{rahimi2025agnosticfedavg}.
\begin{theorem}[\textbf{Convergence of \textbf{FedAVOT}}]
\label{thm:main}
Under Assumptions~\ref{assump:convex}--\ref{assump:gradient_bound} 
and the feasibility condition of Theorem~\ref{thm:feasibility} that ensures existence of a transport map, 
\textbf{FedAVOT} with stepsize $\eta=\Theta(1/\sqrt{TH})$ satisfies
\[
\boxed{
\mathbb{E}\!\left[f\!\left(\dfrac{1}{T}\sum_{t=1}^T \hat{\theta}_{tH}\right)-f(\theta^*)\right]
= \mathcal{O}\!\left(\dfrac{1}{\sqrt{T}}\right).
}
\]
\end{theorem}

We note that the rate in Theorem~\ref{thm:main} is \textit{independent of the number of active users per round}. Thus, \textit{provided that the feasibility conditions of Theorem~\ref{thm:feasibility} hold}, \textbf{FedAVOT} achieves $\mathcal{O}(1/\sqrt{T})$ convergence even when each round involves as few as two participants (albeit with possibly larger variance). In infeasible regimes, IPFP converges to a KL-projected distribution $\tilde{p}$~\cite{csiszar1975ipr}, introducing a non-vanishing bias term (full analysis is deferred to the journal version).

Overall, \textbf{FedAVOT} yields an aggregation rule that simultaneously respects availability ($q$) while optimizing for importance ($p$). By framing aggregation as a MOT problem, \textbf{FedAVOT} inherits both theoretical guarantees and practical implementability. Notably, even when the number of active users per round is very small---as few as two---our analysis (and experiments; see below) confirm that \textbf{FedAVOT} retains the same convergence rate as classical FL methods~\cite{li2020federated, mohri2019agnostic}, providing strong communication efficiency and robustness guarantees under severe participation constraints. 
\vspace{-9pt}
\section{Experiments}
\vspace{-7pt}

We consider a FL setup with $N=100$ users, each holding a heterogeneous local dataset so that the overall system departs significantly from the \textit{iid} assumption. Heterogeneity is incorporated differently across \textit{two tasks}: \textit{Linear regression}, where each user receives samples from distinct feature distributions (Gaussian with user-specific mean and variance), and  \textit{MNIST classification}, where each user is assigned only a small subset of two digits, producing strong label skew. Two distinct settings are investigated. In the \emph{restricted availability} setting (regression), the global importance distribution $p$ is chosen proportional to $(-i)$, giving higher weight to users with smaller indices, while the \textit{availability prior} $r$ is proportional to $(i)$, favoring frequent selection of large-index users. Subsets of size $K=2$ are sampled without replacement from $r$, and the resulting distribution $q \in \Delta^{M-1}$ over $M=\binom{N}{2}$ pairs deviates sharply from $p$, creating a pronounced distribution shift. In the \emph{coordinated} (server-controlled) setting (MNIST classification), the importance weights are instead taken as $p_i \propto \exp(-i/10)$ to induce a heavy skew, while the availability is uniform across users. In this case, $q$ is uniform over $K=2$-subsets, representing the simplest and most generic server-controlled sampling protocol.

The key algorithms compared differ essentially in the aggregation step. Their update rules can be written as:

\vspace{-2ex}
\begin{equation}
\begin{aligned}\nonumber
    \text{FedAvg(full)}: \quad & \hat{\theta}^{(t)} = \sum_{i=1}^N p_i \, \theta_i^{(t-1)}, \\[-0.35em]
    \text{FedAvg($K$)}: \quad & \hat{\theta}^{(t)} = \sum_{i \in S^t} \tfrac{N}{K}\, p_i \, \theta_i^{(t-1)}, \\[-0.35em]
    \text{\textbf{FedAVOT}}: \quad & \hat{\theta}^{(t)} = \sum_{i \in S^t} Y[i,j(t)] \, \theta_i^{(t-1)}.
\end{aligned}
\end{equation}
FedAvg with full participation exactly matches the minimizer of the global objective, but requires communication from all users at every round. FedAvg with partial participation reduces communication by selecting only $K$ users and upscaling their contributions, which is unbiased under uniform availability but leads to amplitude distortion\footnote{One can see that $\mathbb{E}_{S^t}[\frac{N}{K}\sum_{i\in S^t}p_i]\neq1 $. For more details check \href{https://github.com/HerlockSholmesm/FedAVOT.git}{GitHub}.} and oscillatory behavior if $p \neq (1/N)\mathbf{1}_{N}$. \textbf{FedAVOT} overcomes this issue by solving for a transport plan $T$ such that $Tq = p$, ensuring the expected update respects the importance distribution, thus stabilizing convergence even under severe distribution shift.

The results are reported in Figure~\ref{fig:experiments}. In the restricted availability setting for linear regression (Fig.~\ref{fig:experiments}(a)), FedAvg with partial participation fails to reduce the global loss because the systematic bias between $p$ and $q$ overwhelms learning. \textbf{FedAVOT}, on the other hand, is able to reconcile the discrepancy and follows almost the same trajectory as full-participation FedAvg, despite using only two users per round. In the coordinated setting on MNIST (Fig.~\ref{fig:experiments}(b)), FedAvg with partial participation again fails, this time manifesting in oscillatory loss behavior caused by the amplitude distortion of the aggregated parameter vector. \textbf{FedAVOT} avoids this instability, converging smoothly and closely matching the performance of full participation.

The significance of our findings is that \textbf{FedAVOT} achieves nearly identical performance to full participation FedAvg, while drastically reducing communication: \textit{even $K=2$ active users at each round can be enough}. In availability-limited regimes, \textbf{FedAVOT} corrects sampling-induced bias, and in coordinated regimes, it removes scaling mismatch that destabilizes partial FedAvg. Thus, \textbf{FedAVOT} combines the statistical efficiency of full participation with the communication efficiency of partial selection. 

All experiments are averaged over five random seeds. In the linear regression tasks, heterogeneity is introduced by assigning users data from different underlying feature–label relations. In the MNIST tasks, label skew is induced by partitioning classes unevenly across users. 
Further implementation details, code, and scripts to reproduce the results can be found in our \href{https://github.com/HerlockSholmesm/FedAVOT.git}{GitHub repository}.

\vspace{-9pt}
\section{Conclusion}
\label{sec:conclusion}
\vspace{-7.5pt}

We have introduced \textbf{Federated AVeraging with Optimal Transport (\textbf{FedAVOT})}, a framework designed to explicitly align the availability and importance distributions in federated learning through a masked optimal transport formulation. The method admits efficient computation via iterative proportional fitting and retains the classical $\mathcal{O}(1/\sqrt{T})$ convergence rate under a nonsmooth convex setup, even at the presence of stringent partial participation with as few as two clients per global round. Empirical studies on linear regression and MNIST classification on hard heterogeneous tasks demonstrate that \textbf{FedAVOT} enhances stability, fairness, and performance relative to FedAvg when availability and importance distributions are even significantly misaligned.  These results establish \textbf{FedAVOT} as a robust and principled approach for communication-limited federated optimization, with future directions including extensions to non-convex models and analysis of relevant risk measures in this setting.
\clearpage
\bibliographystyle{IEEEtran}
\bibliography{main}
\end{document}